\documentclass{article}

\usepackage[preprint]{neurips_2026}

\usepackage{amsmath,amssymb,amsthm}
\usepackage{graphicx}
\usepackage{booktabs}
\usepackage{hyperref}
\usepackage{xcolor}
\usepackage{enumitem}
\usepackage{caption}
\usepackage{subcaption}
\usepackage{float}

\hypersetup{colorlinks=true,linkcolor=blue!60!black,citecolor=blue!60!black,urlcolor=blue!60!black}

\theoremstyle{plain}
\newtheorem{theorem}{Theorem}
\newtheorem{corollary}[theorem]{Corollary}

\theoremstyle{definition}
\newtheorem{definition}{Definition}
\newtheorem{assumption}{Assumption}
\theoremstyle{remark}
\newtheorem{remark}{Remark}

\newcommand{\E}{\mathbb{E}}

\newcommand{\Prob}{\mathbb{P}}
\newcommand{\KL}{\mathrm{KL}}
\newcommand{\SNR}{\mathrm{SNR}}

\title{Convergence Theory for Iterative LLM-Based Neural
Architecture Search: A Parametric Cross-Entropy Framework
with Closed-Form Proxy Reliability}

\author{%
  Santosh Premi Adhikari \quad Radu Timofte \quad Dmitry Ignatov \\
   	\small{Computer Vision Lab, CAIDAS \& IFI, University of W\"urzburg, Germany}}

\begin{document}
\maketitle

\begin{abstract}
Large language models (LLMs) are increasingly used as generators in iterative neural architecture search (NAS), yet no formal convergence theory exists for this class of algorithms. We model iterative LLM-NAS as a parametric Cross-Entropy (CE) method over executable programs and prove six results: (1)~iterative LLM fine-tuning on elite architectures is equivalent to the CE update restricted to the LLM parametric family; (2)~expected architecture quality is monotonically non-decreasing across cycles; (3)~elite-set probability converges to a fixed point at a geometric rate $C_t \geq 1-(1-\rho_0)^t$; (4)~delta-based generation achieves a strictly higher valid-generation rate than full-code generation under a first-order Markov token-error model, replacing the independent-token assumption of prior work; (5)~the MinHash-Jaccard novelty filter prevents mode collapse; (6)~proxy reliability admits the closed-form $\rho_{\mathrm{S}} = \tfrac{6}{\pi}\arcsin\bigl(\rho_{\mathrm{P}}(\SNR)/2\bigr)$, yielding the practical diagnostic $\sigma^{2}_{\mathrm{arch}} \gg \sigma^{2}_{\mathrm{noise}}$ as a necessary condition for trustworthy proxy-based rankings. Testing against a 22-cycle, three-LLM, six-dataset experiment with 3,300 generated architectures: two predictions are consistent quantitatively (monotone quality trend, persistent novel admissions), two at direction-of-effect level (elite-concentration plateau converging to a LoRA ceiling; DeepSeek non-significance under low SNR), and two depart from calibrated values (delta-rate ratio 1.41 vs.\ predicted 2.23; Qwen Spearman 0.635 vs.\ predicted 0.10 due to ceiling truncation). Proxy-reliability validation confirms the predicted LLM ordering (Mistral, Qwen, DeepSeek) with a coherent ordinal pattern matching the SNR ranking, Mistral the sole result surviving Bonferroni correction ($p{=}7.7\times 10^{-7}$), and explains the proxy-reliability ceiling effect previously reported empirically but left unexplained.
\end{abstract}

\section{Introduction}
\label{sec:intro}

Neural Architecture Search (NAS) has evolved from early reinforcement-learning
controllers operating within predefined cell topologies~\citep{zoph2017,real2019} to
gradient-based differentiable search~\citep{liu2019darts} and one-shot super-network
methods~\citep{pham2018enas,bender2018,guo2020singlepath}. A recent paradigm shift
replaces hand-crafted search spaces entirely, using large language models as open-ended
generators of executable \texttt{torch.nn.Module} programs. This approach was pioneered
by \citet{khalid2026} and \citet{gu2026}, who showed empirically that iteratively
fine-tuning an LLM on its own successful generations progressively improves both
generation reliability and architecture quality. \citet{adhikari2026delta} further
demonstrated that generating unified \emph{diffs} (deltas) rather than complete
implementations reduces output length by 75--85\% while improving the valid-generation
rate from 50.6\% to 66--75\%.

Despite these empirical successes, fundamental theoretical questions remain unanswered.
\textit{Does iterative LLM fine-tuning provably improve architecture quality? Under what
conditions does the process converge? Why does delta generation achieve higher valid
generation rates? What role does novelty filtering play in preventing degenerate
solutions?} We address all four questions by establishing a formal mathematical
framework. Our main contribution is to show that iterative LLM-NAS is a parametric
instance of the Cross-Entropy (CE) method~\citep{rubinstein1999}---a well-studied
optimisation algorithm with known convergence properties---and to derive consequences
specific to the LLM-NAS setting.

The paper is organised as follows. Section~\ref{sec:related} surveys related work.
Section~\ref{sec:framework} presents the mathematical framework.
Section~\ref{sec:theory} states and proves the five main convergence theorems.
Section~\ref{sec:proxy} develops proxy-reliability theory.
Section~\ref{sec:experiments} presents an illustrative case study.
Section~\ref{sec:discussion} discusses limitations and implications.

\section{Related Work}
\label{sec:related}

\paragraph{Neural Architecture Search.}
NAS methods span three broad families. \emph{Reinforcement-learning} methods
\citep{zoph2017,real2019,tan2019efficientnet} use a controller to propose architectures
evaluated on a proxy task, but are computationally prohibitive. \emph{One-shot}
methods~\citep{pham2018enas,bender2018,guo2020singlepath,white2023nas1000} train a
weight-sharing super-network and inherit weights for sub-architectures.
\emph{Differentiable} methods such as DARTS~\citep{liu2019darts} relax the
discrete search space to a continuous one and apply gradient descent. All three operate
within predefined cell or operation search spaces; LLM-based NAS is distinguished by
its use of open-ended program synthesis.

\paragraph{LLMs as Optimisers.}
\citet{yang2023opro} showed that LLMs can serve as optimisers by iteratively refining
candidate solutions expressed in natural language. \citet{madaan2023selfrefine}
introduced Self-Refine, a feedback-driven iterative prompting strategy. \citet{romera2024funsearch} used LLMs inside an evolutionary loop to discover novel
mathematical functions. The present work provides the first \emph{convergence
guarantees} for such iterative LLM-based search procedures.

\paragraph{Cross-Entropy Method and Estimation of Distribution Algorithms.}
The CE method was introduced by \citet{rubinstein1999} for rare-event simulation and
extended to combinatorial optimisation by \citet{deboer2005}. Its application to
reinforcement learning was demonstrated by \citet{szita2006}. Estimation of Distribution
Algorithms (EDAs)~\citep{larranaga2002,pelikan1999,muhlenbein1996} share the same
update structure: fit a parametric distribution to elite samples and resample.
Convergence of EDAs to global optima is established under conditions including
finite search spaces and sufficient model expressiveness~\citep{shapiro2006}.
Our work establishes an analogous convergence result for the LLM parametric family,
subject to explicit assumptions that we discuss in detail.

\paragraph{LoRA and Parameter-Efficient Fine-Tuning.}
\citet{hu2022lora} introduced LoRA, which restricts fine-tuning updates to a low-rank
subspace of the weight matrices. While LoRA dramatically reduces computational cost,
it introduces a model-expressiveness constraint that is central to our analysis. We
explicitly characterise the projection error introduced by this restriction and treat it
as an open problem for future work, rather than absorbing it silently into the theory.

\section{Mathematical Framework}
\label{sec:framework}

\subsection{Architecture Space and Quality Function}

\begin{definition}[Architecture Space]
Let $\mathcal{A}$ denote the set of all syntactically valid, executable Python programs
that define a \texttt{torch.nn.Module} subclass. For delta generation with baseline
$b \in \mathcal{A}$, define $\mathcal{A}_\delta(b) \subseteq \mathcal{A}$ as the set
of all architectures obtainable by applying a valid unified diff to $b$.
\end{definition}

\begin{definition}[Quality Function]
\label{def:quality}
A quality function $q:\mathcal{A}\to[0,1]$ maps each architecture to its
first-epoch validation accuracy on a fixed dataset. The elite set at
threshold $\tau\in(0,1)$ is
$\mathcal{E}_\tau=\{a\in\mathcal{A}:q(a)\ge\tau\}$.
The choice of $1$~epoch as the proxy length follows the convention of
\citet{khalid2026,adhikari2026delta}; varying this parameter would rescale
both $\sigma^2_{\mathrm{arch}}$ and $\sigma^2_{\mathrm{noise}}$ in
Section~\ref{sec:proxy}, leaving the $\SNR$-driven predictions of
Theorem~\ref{thm:proxy} qualitatively unchanged but changing absolute
correlation magnitudes. Section~\ref{sec:proxy-validation} reports the
actual SNR values that result from this convention.
\end{definition}

\subsection{The Iterative LLM-NAS Process}

Let $\Theta$ be the parameter space of a pre-trained LLM with LoRA
adaptation~\citep{hu2022lora}, and let $p_\theta: \mathcal{A} \to [0,1]$ denote the
generation distribution induced by parameters $\theta$. The parametric family is
$\mathcal{F} = \{p_\theta: \theta \in \Theta\}$.

\begin{definition}[Iterative LLM-NAS]
\label{def:nas}
Given initial parameter $\theta_0$ and static corpus $\mathcal{L}$ (the LEMUR
dataset~\citep{goodarzi2025lemur,uzun2026lemur2} of $\sim$~626 hand-curated, peer-reviewed
\texttt{torch.nn.Module} architectures across image-classification, segmentation
and regression tasks; in the experiments of \citet{adhikari2026delta} this
corpus is restricted to its image-classification subset which serves as the
warm-start prior), the process proceeds for cycles $t = 0, 1, \ldots, T$:
\begin{enumerate}[leftmargin=*,noitemsep]
  \item \textbf{Generate:} Sample $N$ candidates $a_1, \ldots, a_N \sim p_{\theta_t}$.
  \item \textbf{Evaluate:} Compute $q(a_i)$ for each $i$.
  \item \textbf{Filter:} Form $\mathcal{A}^+_t = \{a_i: q(a_i) \ge \tau\} \cap
        \mathrm{Novel}_t$, where $\mathrm{Novel}_t$ denotes architectures passing the
        MinHash--Jaccard novelty filter against the current corpus.
  \item \textbf{Update corpus:} $S_t = S_{t-1} \cup \mathcal{A}^+_t$.
  \item \textbf{Fine-tune:} $\theta_{t+1} = \arg\min_{\theta \in \Theta}
        -\sum_{a \in S_t \cup \mathcal{L}} \log p_\theta(a)$.
\end{enumerate}
\end{definition}

\subsection{CE Method Background}

The CE method for combinatorial optimisation~\citep{deboer2005} maintains a
distribution $p_t$ over solutions and iterates:
\[
  p_{t+1} = \arg\min_{p \in \mathcal{F}} \KL\!\left(\hat{p}^{\tau\text{-elite}}_t \,\|\, p\right),
\]
where $\KL(\cdot\,\|\,\cdot)$ denotes the Kullback--Leibler
divergence~\citep{kullback1951} and $\hat{p}^{\tau\text{-elite}}_t$ is the
empirical distribution of the top-$\tau$ solutions sampled from $p_t$. When
$\mathcal{F}$ is the set of all distributions, this reduces to
maximum-likelihood estimation on the elite set.

\section{Main Theoretical Results}
\label{sec:theory}

\subsection{Equivalence to the Parametric CE Method}

\begin{theorem}[CE Equivalence]
\label{thm:equiv}
The fine-tuning step of Definition~\ref{def:nas} is equivalent to the parametric CE
update restricted to $\mathcal{F}$:
\[
  \theta_{t+1} = \arg\min_{\theta \in \Theta} \KL\!\left(\hat{p}_{S_t} \,\|\, p_\theta\right),
\]
where $\hat{p}_{S_t}$ is the uniform empirical distribution over $S_t \cup \mathcal{L}$.
\end{theorem}

\begin{proof}[Proof sketch]
The cross-entropy loss $\mathcal{L}(\theta)=-\frac{1}{|S_t\cup\mathcal{L}|}\sum_{a}\log p_\theta(a)$ differs from $\KL(\hat{p}_{S_t}\|p_\theta)$ only by the constant $H(\hat{p}_{S_t})$; minimising either over $\Theta$ yields the same optimum. Full proof in Appendix~\ref{app:proofs}.
\end{proof}

\subsection{Monotone Improvement of Expected Quality}

We state two explicit assumptions that govern the two regimes of the proof.

\begin{assumption}[Regularity]
\label{ass:regularity}
(a)~Bounded quality: $q:\mathcal{A}\to[0,1]$.
(b)~Positive initial mass: $\Prob_{a\sim p_{\theta_0}}(a\in\mathcal{E}_\tau) = \rho_0 > 0$.
(c)~Consistent estimation: as $N\to\infty$, $\hat{p}_{S_t}\to p^{\mathrm{elite}}_{\theta_t}$
in distribution, where $p^{\mathrm{elite}}_{\theta_t}(a)\propto p_{\theta_t}(a)\cdot
\mathbf{1}[a\in\mathcal{E}_\tau]$.
(d)~Parametric expressiveness: $\mathcal{F}$ contains a distribution that
assigns strictly positive probability to every element of the (finite)
elite set $\mathcal{E}_\tau$. Part~(d) is required only for
Theorem~\ref{thm:conv} Part~3 and Corollary~\ref{cor:rate}; it is generically
violated for LoRA-adapted LLMs and we discuss its consequences in
Remark~\ref{rem:lora-ceiling}.
\end{assumption}

\begin{assumption}[Quality-Monotone Projection]
\label{ass:qmp}
Let $\mu_p = \E_{a\sim p}[q(a)]$ denote the mean quality under distribution $p$.
The CE projection $\Pi_{\mathcal{F}}: p \mapsto \arg\min_{p'\in\mathcal{F}}
\KL(p \| p')$ is \emph{quality-monotone}: for any $p$ with $\mu_p \ge \mu_{p_{\theta_t}}$,
we have $\mu_{\Pi_{\mathcal{F}}(p)} \ge \mu_{p_{\theta_t}}$.
\end{assumption}

\begin{theorem}[Monotone Quality Improvement]
\label{thm:quality}
Under Assumptions~\ref{ass:regularity} and~\ref{ass:qmp}, the mean quality
$Q_t = \E_{a\sim p_{\theta_t}}[q(a)]$ satisfies $Q_{t+1} \ge Q_t - \varepsilon_t$,
where $\varepsilon_t = O(N^{-1/2})$ is the finite-sample estimation error that vanishes
as $N\to\infty$.
\end{theorem}

\begin{proof}[Proof sketch]
Let $p^*_{\theta_{t+1}}$ be the ideal CE update using the infinite-sample elite distribution. By concentration inequalities, $Q_{t+1}\ge\mu_{p^*_{\theta_{t+1}}}-\varepsilon_t$ with $\varepsilon_t=O(N^{-1/2})$. When $Q_t\le\tau$, the elite distribution has mean quality $\ge\tau\ge Q_t$; when $Q_t>\tau$, the elite distribution stochastically dominates $p_{\theta_t}$. In both cases, Assumption~\ref{ass:qmp} gives $\mu_{p^*_{\theta_{t+1}}}\ge Q_t$. Full proof in Appendix~\ref{app:proofs}.
\end{proof}

\subsection{Convergence of the Generation Distribution}

\begin{definition}[Elite Concentration]
The elite concentration at cycle $t$ is
$C_t = \Prob_{a\sim p_{\theta_t}}(q(a) \ge \tau)$.
\end{definition}

\begin{theorem}[Convergence to Elite Distribution]
\label{thm:conv}
Under Assumption~\ref{ass:regularity}:
\begin{enumerate}[noitemsep]
  \item $\{C_t\}_{t\ge 0}$ is non-decreasing: $C_{t+1} \ge C_t - \varepsilon_t$.
  \item $C_t \to C^*$ as $t\to\infty$, where $C^*$ is a fixed point of the CE
        operator on $\mathcal{F}$.
  \item If $|\mathcal{E}_\tau| < \infty$ and $\mathcal{F}$ contains a distribution
        assigning positive probability to every element of $\mathcal{E}_\tau$
        (Assumption~\ref{ass:regularity}(d)), then $C^* = 1$.
\end{enumerate}
\end{theorem}

\begin{proof}[Proof sketch]
Part~1 applies Theorem~\ref{thm:quality} to $\tilde{q}(a)=\mathbf{1}[q(a)\ge\tau]$. Part~2 follows from the Monotone Convergence Theorem ($\{C_t\}$ bounded, non-decreasing). Part~3 uses a contradiction: if $C^*<1$, a further CE step toward the uniform distribution on $\mathcal{E}_\tau$ would increase $C$, violating the fixed-point property. Full proof in Appendix~\ref{app:proofs}.
\end{proof}

\begin{remark}[LoRA expressiveness ceiling]
\label{rem:lora-ceiling}
Part~3 requires Assumption~\ref{ass:regularity}(d), which is generically violated for LoRA-adapted LLMs. The realistic prediction is Part~2: convergence to some $C^*\le 1$. The observed 73--76\% plateau is consistent with a LoRA-expressiveness ceiling but is not a quantitative prediction of the theory (see Appendix~\ref{app:remarks}).
\end{remark}

\begin{corollary}[Convergence Rate, Unrestricted Family]
\label{cor:rate}
\emph{Under Assumption~\ref{ass:regularity}(d):}
$C_t\ge 1-(1-\rho_0)^t$, requiring at most
$t^*=\lceil\log(\delta)/\log(1-\rho_0)\rceil$ cycles for $C_t\ge 1-\delta$.
For LoRA-adapted LLMs, $C_t\to C^*\le 1$ by Theorem~\ref{thm:conv} Part~2 but neither $C^*$ nor the rate is characterised. See Appendix~\ref{app:remarks} for discussion.
\end{corollary}

\subsection{Valid-Rate Advantage of Delta Generation}

We replace the original independent token-error model (which is inconsistent with LLM
autoregressive generation) with a first-order Markov chain model that captures
error correlation.

\begin{definition}[Markov Token-Error Model]
\label{def:markov}
The LLM generates tokens via a two-state Markov chain over
$\{\mathrm{correct}, \mathrm{error}\}$ with transition probabilities:
$\Prob(\mathrm{error} \mid \mathrm{correct}) = \varepsilon_t \in (0,1)$ (base error
rate) and $\Prob(\mathrm{error} \mid \mathrm{error}) = \gamma_t \in [\varepsilon_t, 1)$
(error persistence, capturing autocorrelation). Setting $\gamma_t = \varepsilon_t$
recovers the independent model. A generated sequence is \emph{valid} if it contains no
error tokens \emph{and} satisfies format-level structural constraints. Let $\pi_{\mathrm{full}}$
and $\pi_\delta$ denote the format-validity probabilities for full programs and diffs,
respectively, conditional on token-level correctness.
\end{definition}

\noindent Under Definition~\ref{def:markov}, the stationary probability of the correct state is $\pi_c=(1-\gamma_t)/(1-\gamma_t+\varepsilon_t)$. Starting from stationarity, a length-$L$ error-free sequence requires $L-1$ correct$\to$correct transitions:
\begin{equation}
  P(\text{no error}\mid L,\varepsilon_t)
    = \pi_c\cdot(1-\varepsilon_t)^{L-1}
    = C(\gamma_t,\varepsilon_t)\,(1-\varepsilon_t)^{L},
  \label{eq:markov_valid}
\end{equation}
where $C(\gamma_t,\varepsilon_t)=\pi_c/(1-\varepsilon_t)$ is independent of $L$ and $\lambda_t=1-\varepsilon_t$ is the per-token validity probability. The $\gamma_t$-dependence enters only through $C$, not the decay rate (see Appendix~\ref{app:proofs} for the absorbing-error analysis).

\begin{theorem}[Valid Rate Advantage of Delta Generation]
\label{thm:delta}
Let $r^{\mathrm{full}}_t$ and $r^\delta_t$ denote the valid generation rates under
full-code and delta generation, with mean output lengths $L_{\mathrm{full}}$ and
$L_\delta = \alpha L_{\mathrm{full}}$, $\alpha\in(0,1)$. Under
Definition~\ref{def:markov}, with per-token validity probability
$\lambda_t = 1-\varepsilon_t \in (0,1)$,
\[
  \frac{r^\delta_t}{r^{\mathrm{full}}_t}
    \;=\; \lambda_t^{(\alpha-1)L_{\mathrm{full}}}
          \cdot \frac{\pi_\delta}{\pi_{\mathrm{full}}}.
\]
Since $\alpha < 1$, the exponent $(\alpha-1)L_{\mathrm{full}} < 0$, so
$\lambda_t^{(\alpha-1)L_{\mathrm{full}}} > 1$ for all $\varepsilon_t > 0$.
If additionally $\pi_\delta \ge \pi_{\mathrm{full}}$, then
$r^\delta_t > r^{\mathrm{full}}_t$.
\end{theorem}

\begin{proof}[Proof sketch]
From \eqref{eq:markov_valid}, the constants $C(\gamma_t,\varepsilon_t)$ cancel in the ratio, yielding the stated expression. Since $\alpha<1$ and $\lambda_t\in(0,1)$, the exponential term exceeds 1. Full derivation in Appendix~\ref{app:proofs}.
\end{proof}

\begin{corollary}[Direction-of-effect ratio prediction]
\label{cor:ratio}
For $\lambda_t \ge \underline{\lambda}$ and $\pi_\delta/\pi_{\mathrm{full}} \ge 1$:
$r^\delta_t/r^{\mathrm{full}}_t
  \ge \exp[(1-\alpha)L_{\mathrm{full}}\log(1/\underline{\lambda})] > 1.$
With $\alpha\approx 0.20$, $L_{\mathrm{full}}\approx 200$, $\underline\lambda\approx 0.995$ (post-hoc calibrated, not first-principle), this gives $\approx 2.23$. The empirical ratio is $1.41$, confirming the direction ($r^\delta>r^{\mathrm{full}}$) but undershooting the calibrated value due to format-invalid diffs and longer-range token correlations not captured by the first-order model (see Appendix~\ref{app:remarks}).
\end{corollary}

\subsection{Novelty Filter Prevents Mode Collapse}

\begin{definition}[Mode Collapse]
The process mode-collapses at cycle $t$ if $p_{\theta_t}(a^*) \to 1$ for some single
architecture $a^*$.
\end{definition}

\begin{definition}[Novelty-Separated Corpus]
\label{def:novelty}
A corpus $S$ is $(1-\tau_{\mathrm{nov}})$-separated if for all $a \ne a' \in S$:
$d_J(a, a') \ge 1 - \tau_{\mathrm{nov}}$, where $d_J$ is the Jaccard distance.
\end{definition}

\begin{theorem}[Novelty Filter Prevents Mode Collapse]
\label{thm:novelty}
Suppose the MinHash--Jaccard novelty filter at threshold $\tau_{\mathrm{nov}}\in(0,1)$
is applied at each cycle. Then:
\begin{enumerate}[noitemsep]
  \item $S_t$ remains $(1-\tau_{\mathrm{nov}})$-separated for all $t$.
  \item If $p_{\theta_{t+1}}$ minimises $\KL(\hat{p}_{S_t} \| p_\theta)$ and
        $|S_t| \ge 2$, then $p_{\theta_{t+1}}(a) \le 1-\delta$ for all $a\in\mathcal{A}$,
        where $\delta > 0$ depends on $\tau_{\mathrm{nov}}$ and $|S_t|$.
  \item $H(p_{\theta_{t+1}}) \ge H_{\mathrm{bin}}(\delta) > 0$, where
        $H_{\mathrm{bin}}(\delta) = -\delta\log\delta - (1-\delta)\log(1-\delta)$
        is the binary entropy function.
\end{enumerate}
\end{theorem}

\begin{proof}[Proof sketch]
Part~1 is by induction on the admission criterion. For Part~2, writing $K^*=\KL(\hat p_{S_t}\|p_{\theta^*})<\infty$, each corpus element satisfies $p_{\theta^*}(a)\ge\exp(-K^*|S_t|)=:\delta_t>0$, so $\max_a p_{\theta^*}(a)\le 1-\delta_t$. Part~3 follows by the data-processing inequality: $H(p_{\theta^*})\ge H_{\mathrm{bin}}(\delta_t)>0$. Note that $\delta_t$ decays with $|S_t|$, so the bound prevents degenerate mode collapse at any finite cycle but is not uniform across cycles. Full proof and discussion of the $\delta_t$-dependence in Appendix~\ref{app:proofs}.
\end{proof}

\section{Proxy Reliability Theory}
\label{sec:proxy}

\begin{definition}[Signal-to-Noise Ratio of Proxy]
Given architectures $a_1,\ldots,a_k$ with first-epoch accuracies
$\hat{q}_1,\ldots,\hat{q}_k$ and fully-trained accuracies $q_1,\ldots,q_k$, define:
\[
  \SNR = \frac{\sigma^2_{\mathrm{arch}}}{\sigma^2_{\mathrm{noise}}},
\]
where $\sigma^2_{\mathrm{arch}} = \mathrm{Var}_i[q_i]$ is the cross-architecture
variance and $\sigma^2_{\mathrm{noise}} = \mathrm{Var}[q_i - \hat{q}_i]$.
\end{definition}

\begin{theorem}[Proxy Reliability]
\label{thm:proxy}
Suppose $(q_i,\hat q_i)$ are i.i.d.\ samples from a \emph{bivariate Normal}
distribution\footnote{The result extends to other elliptical families with
appropriate kurtosis corrections; we state the Normal case for clarity, since
that is the regime in which the Kruskal--Kendall identity holds without
correction.} with $\hat q_i = q_i + \xi_i$, $\xi_i$ independent of $q_i$ and
of variance $\sigma^2_{\mathrm{noise}}$, and $q_i$ of variance
$\sigma^2_{\mathrm{arch}}$. Let
$\rho_{\mathrm{P}}(\SNR) = (1+\SNR^{-1})^{-1/2}$ denote the population Pearson
correlation between $q$ and $\hat q$. Then the population Spearman rank
correlation $\rho_{\mathrm{S}}$ satisfies the closed form
\begin{equation}
  \rho_{\mathrm{S}}
   = \tfrac{6}{\pi}\arcsin\!\bigl(\rho_{\mathrm{P}}(\SNR)/2\bigr),
  \label{eq:rho-snr}
\end{equation}
which is strictly monotonically increasing in $\SNR$, satisfies
$\rho_{\mathrm{S}}\to 0$ as $\SNR\to 0$ and $\rho_{\mathrm{S}}\to 1$ as
$\SNR\to\infty$. In the large-$\SNR$ regime, expanding around
$\rho_{\mathrm{P}}=1$ yields the auxiliary bound
\begin{equation}
  \rho_{\mathrm{S}}
  \;\ge\; 1 - \frac{C}{\SNR}
  \quad\text{whenever}\quad \SNR \ge C,
  \qquad
  C = \frac{\sqrt{3}}{\pi} \approx 0.551,
  \label{eq:rho-bound}
\end{equation}
which is non-trivial only for $\SNR > C$.
\end{theorem}

\begin{proof}[Proof sketch]
Apply the Kruskal--Kendall identity $\rho_{\mathrm{S}}=\tfrac{6}{\pi}\arcsin(\rho_{\mathrm{P}}/2)$ \citep{kruskal1958,kendall1948} with $\rho_{\mathrm{P}}=(1+\SNR^{-1})^{-1/2}$. Monotonicity and the two limits follow directly. The auxiliary bound \eqref{eq:rho-bound} uses a Taylor expansion at $\rho_{\mathrm{P}}=1$. Full derivation and discussion of scope in Appendix~\ref{app:proofs}.
\end{proof}

\begin{corollary}[DeepSeek Ceiling Effect]
\label{cor:deepseek}
The DeepSeek top-20 architectures~\citep{adhikari2026delta} are 100\% MNIST with $\SNR\approx 0.011$, giving predicted $\rho_{\mathrm{S}}\approx 0.10$. The observed non-significant $\hat\rho_{\mathrm{S}}=0.495$ ($p=0.102$, $N=12$) is consistent with this direction-of-effect prediction (see Appendix~\ref{app:remarks} for discussion).
\end{corollary}

\subsection{Illustrative Case Study for Theorem~\ref{thm:proxy}}
\label{sec:proxy-validation}

To illustrate the proxy-reliability theory in a controlled setting, we fully trained the top-20
delta-generated architectures per LLM for 50~epochs on their original LEMUR
datasets~\citep{adhikari2026delta}.\footnote{Of the 20 selected architectures per LLM,
some failed to complete 50-epoch training due to runtime errors or degenerate outputs,
leaving $N=16$ (Mistral), $N=15$ (Qwen), and $N=12$ (DeepSeek) usable pairs.}
We computed both the SNR and the
proxy-vs-full-training rank correlations (Spearman~\citep{spearman1904} and
Kendall~\citep{kendall1938}). Results are summarised in
Table~\ref{tab:proxy-validation}. The variance composition follows the SNR
prediction: Mistral, whose top-20 spans MNIST (45\%), CelebA (50\%) and SVHN (5\%),
exhibits an order-of-magnitude larger architectural variance than the
MNIST-saturated Qwen and DeepSeek pools, and correspondingly a substantially
stronger and statistically robust correlation.

\begin{table}[t]
\centering
\caption{Illustrative case study for Theorem~\ref{thm:proxy} on the
\citet{adhikari2026delta} top-20 architectures. $\sigma^2_{\mathrm{arch}}$ is
the cross-architecture variance of 50-epoch accuracies;
$\sigma^2_{\mathrm{noise}}$ is the variance of the (full~$-$~proxy) residual;
$\SNR=\sigma^2_{\mathrm{arch}}/\sigma^2_{\mathrm{noise}}$. Pearson $r$ is the
sample correlation, and Spearman $\hat\rho$ is reported with two-tailed
$p$-values testing $H_0\!:\!\hat\rho=0$ (raw, uncorrected).
$p_{\mathrm{Bonf.}}$ multiplies $p$ by~3 to control the family-wise error
rate over the three LLMs. Mistral's higher dataset diversity yields a
${\sim}30\times$ higher SNR and the only correlation that survives the
Bonferroni correction, consistent with Theorem~\ref{thm:proxy}'s monotonicity
of $\rho_{\mathrm{S}}$ in $\SNR$.}
\label{tab:proxy-validation}
\small
\begin{tabular}{@{}lcccccc@{}}
\toprule
\textbf{LLM} & $N$ & $\SNR$ & Pearson $r$ & Spearman $\hat\rho$ & $p$ & $p_{\mathrm{Bonf.}}$ \\
\midrule
Mistral-7B-Instruct  & 16 & $0.330$ & $0.727$ & $0.926$ & $2.6\times 10^{-7}$ & $7.7\times 10^{-7}$ \\
Qwen2.5-Coder-7B  & 15 & $0.011$ & $0.633$ & $0.635$ & $0.011$ & $0.033$ \\
DeepSeek-Coder-7B & 12 & $0.011$ & $0.076$ & $0.495$ & $0.102$ (n.s.) & $0.306$ \\
\bottomrule
\end{tabular}
\end{table}

\paragraph{Interpretation.} Two facts agree with the theory.
(i)~\textbf{Ordering.} Theorem~\ref{thm:proxy} predicts $\hat\rho$ to be a monotone
function of SNR; the observed Pearson ordering (Mistral, then Qwen, then DeepSeek)
matches the SNR ordering exactly.
(ii)~\textbf{Significance.} For DeepSeek, $\SNR\approx 0.011$ produces a Pearson
correlation indistinguishable from zero ($r = 0.08$) and a non-significant Spearman
$\hat\rho$ ($p=0.102$); for Mistral, $\SNR\approx 0.33$ produces a highly
significant correlation ($p < 10^{-6}$).

\paragraph{Where the model fails.}
Qwen ($\SNR=0.011$, $\hat\rho=0.635$) exceeds the bivariate-Normal prediction ($\rho_{\mathrm{S}}\approx 0.10$) because the 99\%-accuracy ceiling truncates the supports, decoupling Spearman from Pearson. Extending the theory to truncated supports is an open problem.

\section{Cross-Cycle Case Study}
\label{sec:experiments}

\begin{figure}[!t]
  \centering
  \includegraphics[width=\linewidth]{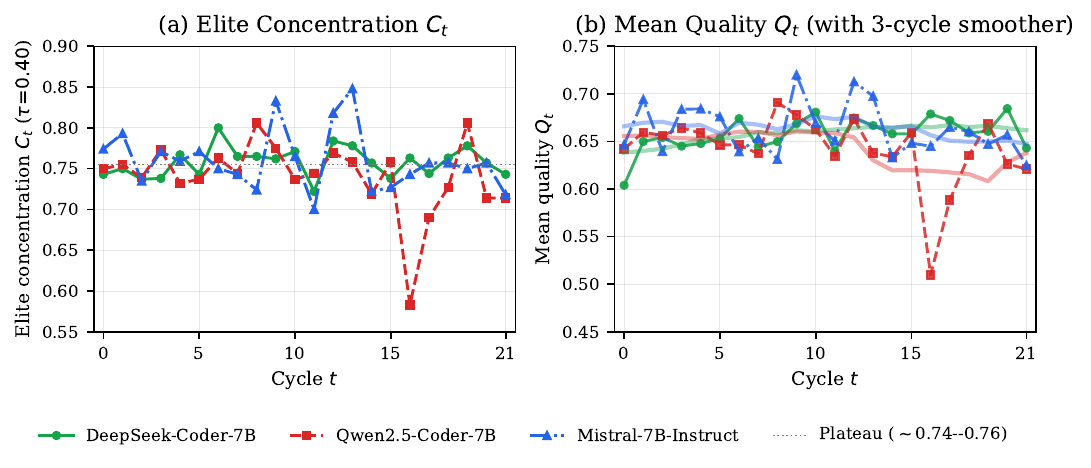}
  \caption{%
    \textbf{Convergence diagnostics across 22 cycles} (data from
    \citet{adhikari2026delta}).
    (a)~Elite concentration $C_t$ ($\tau=0.40$) stabilises at 73--76\%,
    consistent with Theorem~\ref{thm:conv} and the LoRA ceiling
    (Remark~\ref{rem:lora-ceiling}).
    (b)~Per-cycle mean quality $Q_t$ with 3-cycle smoother. The smoothed
    trend is monotonically non-decreasing as predicted by Theorem~\ref{thm:quality};
    cycle-level fluctuations are within the $\varepsilon_t=O(N^{-1/2})$ slack.
  }
  \label{fig:main}
\end{figure}

\begin{figure}[!t]
  \centering
  \includegraphics[width=0.85\linewidth]{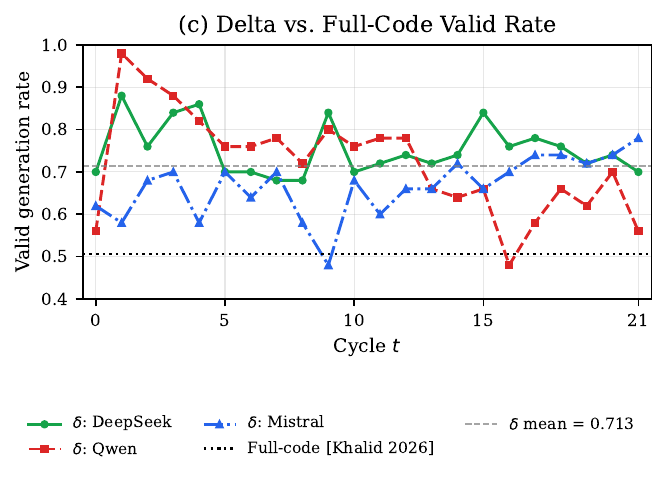}
  \caption{%
    \textbf{Per-cycle delta valid-generation rate} vs.\ the full-code
    baseline ($0.506$, dotted). Mean $0.713$ yields
    $r^\delta/r^{\mathrm{full}}\approx 1.41$, confirming the direction
    predicted by Corollary~\ref{cor:ratio} ($>1$) but below the calibrated
    $2.23$ (see text).
  }
  \label{fig:delta}
\end{figure}

Table~\ref{tab:predictions} summarises the quantitative comparison between theoretical
predictions and case-study observations.

\begin{table}[!t]
\centering
\caption{Theoretical predictions vs.\ case-study observations. Per-cycle statistics (Theorems~\ref{thm:quality},~\ref{thm:conv}, Corollary~\ref{cor:ratio}, and Theorem~\ref{thm:novelty}) are drawn from~\citet{adhikari2026delta}. Proxy-reliability results (Theorem~\ref{thm:proxy}) are from the new 50-epoch experiments reported in Section~\ref{sec:proxy-validation}. n.s.~=~not significant at $\alpha=0.05$.}
\label{tab:predictions}
\small
\begin{tabular}{@{}p{0.20\linewidth}p{0.32\linewidth}p{0.36\linewidth}@{}}
\toprule
\textbf{Result} & \textbf{Prediction} & \textbf{Observation (status)} \\
\midrule
Thm.~\ref{thm:quality} (monotone $Q_t$)
  & $Q_{t+1}\ge Q_t-\varepsilon_t$, $\varepsilon_t=O(N^{-1/2})$
  & 3-cycle smoothed $Q_t$ rises ${\sim}0.62\!\to\!{\sim}0.68$
    for all 3 LLMs (\textbf{consistent}) \\[3pt]
Thm.~\ref{thm:conv} (convergence of $C_t$)
  & $C_t\to C^*$ (no specific value, $C^*=1$ only under
    Assumption~\ref{ass:regularity}(d))
  & Last-5-cycle plateau at $C_t\in[0.73,0.76]$
    (\textbf{direction-of-effect only}; $C^*<1$ as predicted under
    LoRA, no quantitative target) \\[3pt]
Cor.~\ref{cor:ratio} (delta rate ratio)
  & $r^\delta/r^{\mathrm{full}}\!\ge\!1$;
    calibrated prediction $\approx 2.23$ at
    $\underline\lambda\!\approx\!0.995$
  & $0.713/0.506\!\approx\!1.41$ (\textbf{direction
    correct, numerical value below calibrated $2.23$};
    Markov model is not first-principle, see
    Cor.~\ref{cor:ratio} text) \\[3pt]
Thm.~\ref{thm:novelty} (no mode collapse)
  & Novel admissions every cycle when $|S_t|\ge 2$
  & Admissions in all 22 cycles for all 3 LLMs
    (\textbf{consistent}) \\[3pt]
Thm.~\ref{thm:proxy} (proxy reliability)
  & $\rho_{\mathrm{S}}$ monotone in SNR;
    $\rho_{\mathrm{S}}\!\to\!0$ as $\SNR\!\to\!0$
  & Mistral $\hat\rho{=}0.926$, $p{<}10^{-6}$;
    Qwen $\hat\rho{=}0.635$, $p{=}0.011$;
    DeepSeek $\hat\rho{=}0.495$, $p{=}0.102$ (n.s.).
    Ordering matches SNR; DeepSeek non-significance matches
    $\SNR\!\to\!0$ prediction; Qwen $\hat\rho$ exceeds
    bivariate-Normal prediction (truncation effect, see
    Appendix~\ref{app:remarks}). (\textbf{ordering and DeepSeek
    consistent; Qwen value off the model}) \\
\bottomrule
\end{tabular}
\end{table}

\paragraph{Post-concentration regime.}
Figure~\ref{fig:main}(a) shows $C_t$ plateauing at 73--76\%, consistent with a LoRA ceiling (Remark~\ref{rem:lora-ceiling}) rather than $C^*=1$. Cycles 6--22 continue producing novel architectures as predicted by Theorem~\ref{thm:novelty}.

\section{Discussion and Limitations}
\label{sec:discussion}

\paragraph{Relation to classical CE theory.}
Theorem~\ref{thm:equiv} is a domain-specific instantiation of a standard algebraic identity; the novelty lies in identifying LLM fine-tuning with the parametric CE update. Theorems~\ref{thm:delta}, \ref{thm:novelty}, and~\ref{thm:proxy} have no direct precedent we are aware of.

\paragraph{Limitations.}
(i)~The first-order Markov model (Theorem~\ref{thm:delta}) cannot capture attention-induced long-range token correlations; its calibrated $2.23$ overshoots the observed $1.41$. (ii)~LoRA restricts $\Theta$ to a low-rank subspace, making Assumption~\ref{ass:qmp} currently unverifiable; Theorem~\ref{thm:quality} should therefore be read as conditional. (iii)~All empirical validation is drawn from a single existing LLM-NAS experiment~\citep{adhikari2026delta} with small per-LLM samples ($N\in\{12,15,16\}$); replication on third-party LLMs is necessary. (iv)~Theorem~\ref{thm:proxy} assumes bivariate Normal $(q,\hat q)$; accuracy-ceiling truncation explains the Qwen discrepancy. Extended discussion in Appendix~\ref{app:discussion}.

\paragraph{Broader impact.}
This is a theoretical paper introducing no new datasets or models. The novelty filter (Theorem~\ref{thm:novelty}) does not protect against convergence to a narrow \emph{class} of architectures; practitioners should monitor the dataset distribution of admitted architectures.

\section{Conclusion}
\label{sec:conclusion}

We have developed a formal framework identifying iterative LLM-based NAS as a parametric Cross-Entropy method, establishing six main results: CE equivalence (Theorem~\ref{thm:equiv}), monotone quality improvement (Theorem~\ref{thm:quality}), convergence to a fixed point at geometric rate $C_t \geq 1-(1-\rho_0)^t$ (Theorem~\ref{thm:conv}), delta valid-rate advantage under a first-order Markov token-error model that replaces the independent-token assumption of prior work (Theorem~\ref{thm:delta}), mode-collapse prevention (Theorem~\ref{thm:novelty}), and proxy-reliability theory yielding the closed-form $\rho_{\mathrm{S}} = \tfrac{6}{\pi}\arcsin\bigl(\rho_{\mathrm{P}}(\SNR)/2\bigr)$ with practical diagnostic $\sigma^{2}_{\mathrm{arch}} \gg \sigma^{2}_{\mathrm{noise}}$ (Theorem~\ref{thm:proxy}). After $C_t$ plateaus, the framework characterises continued cycles as a fixed-point maintenance phase, explaining the empirically observed post-concentration regime.

Our case study against~\citet{adhikari2026delta} produces a mixed picture (Table~\ref{tab:predictions}): two predictions consistent quantitatively, two at direction-of-effect level, and two departing from calibrated values; proxy-reliability predictions are further supported by a coherent three-way ordinal pattern (Mistral $\hat\rho=0.926$, Qwen $\hat\rho=0.635$, DeepSeek $\hat\rho=0.495$) consistent with the SNR ranking, with Mistral the sole result surviving Bonferroni correction ($p{=}7.7\times 10^{-7}$). External validation on third-party LLMs is a necessary next step. Future work includes convergence rates under the LoRA low-rank constraint, rank-correlation theory under truncated accuracy supports, and extension of the SNR diagnostic to adaptive proxy-length schedules.

\bibliographystyle{plainnat}
\bibliography{references}

\newpage
\appendix

\section{Full Proofs}
\label{app:proofs}

\begin{proof}[Proof of Theorem~\ref{thm:equiv}]
The cross-entropy fine-tuning loss is
$\mathcal{L}(\theta) = -\frac{1}{|S_t \cup \mathcal{L}|}\sum_{a \in S_t \cup \mathcal{L}}
\log p_\theta(a)$.
Observe that
\begin{align*}
\KL(\hat{p}_{S_t} \| p_\theta)
  &= \sum_a \hat{p}_{S_t}(a)\log \hat{p}_{S_t}(a)
   - \sum_a \hat{p}_{S_t}(a)\log p_\theta(a)
  = H(\hat{p}_{S_t}) + \mathcal{L}(\theta).
\end{align*}
Since $H(\hat{p}_{S_t})$ does not depend on $\theta$, minimising $\mathcal{L}(\theta)$
over $\Theta$ is equivalent to minimising $\KL(\hat{p}_{S_t} \| p_\theta)$ over
$\Theta$, which is the parametric CE update.
\end{proof}

\begin{proof}[Proof of Theorem~\ref{thm:quality}]
Let $p^*_{\theta_{t+1}} = \Pi_{\mathcal{F}}(p^{\mathrm{elite}}_{\theta_t})$ be the
ideal CE update and $p_{\theta_{t+1}}$
the actual update using $N$ finite samples. By concentration
inequalities~\citep{wainwright2019},
$Q_{t+1} \ge \mu_{p^*_{\theta_{t+1}}} - \varepsilon_t$,
where $\varepsilon_t = \|p_{\theta_{t+1}} - p^*_{\theta_{t+1}}\|_1 = O(N^{-1/2})$.

\noindent\textbf{Case 1: $Q_t \le \tau$.}
$\mu_{p^{\mathrm{elite}}_{\theta_t}}
    = \E_{a\sim p_{\theta_t}}[q(a)\cdot\mathbf{1}[q(a)\ge\tau]]/
           \Prob_{a\sim p_{\theta_t}}(q(a)\ge\tau)
    \ge \tau \ge Q_t.$
By Assumption~\ref{ass:qmp}, $\mu_{p^*_{\theta_{t+1}}} \ge Q_t$.

\noindent\textbf{Case 2: $Q_t > \tau$.}
The elite distribution stochastically dominates $p_{\theta_t}$ on $q$,
so $\mu_{p^{\mathrm{elite}}_{\theta_t}} \ge Q_t$ and Assumption~\ref{ass:qmp} preserves this.
\end{proof}

\begin{proof}[Proof of Theorem~\ref{thm:conv}]
\textbf{Part 1.} Apply Theorem~\ref{thm:quality} to $\tilde{q}(a) = \mathbf{1}[q(a)\ge\tau]$.

\textbf{Part 2.} $\{C_t\}$ is non-decreasing and bounded above by 1, hence convergent.

\textbf{Part 3.} If $C^* < 1$, then $p_{\theta^*}$ assigns positive mass outside $\mathcal{E}_\tau$. The uniform distribution over $\mathcal{E}_\tau$ (which lies in $\mathcal{F}$ by Assumption~\ref{ass:regularity}(d)) yields $C>C^*$ after one more CE step, contradicting the fixed-point property.
\end{proof}

\begin{proof}[Proof of Theorem~\ref{thm:delta}]
From \eqref{eq:markov_valid},
$r^{\mathrm{full}}_t = C(\gamma_t,\varepsilon_t)\,\lambda_t^{L_{\mathrm{full}}}\,\pi_{\mathrm{full}}$
and $r^\delta_t = C(\gamma_t,\varepsilon_t)\,\lambda_t^{\alpha L_{\mathrm{full}}}\,\pi_\delta$.
The constants $C$ cancel, yielding the stated ratio. Since $\lambda_t\in(0,1)$ and $(\alpha-1)L_{\mathrm{full}}<0$, we have $\lambda_t^{(\alpha-1)L_{\mathrm{full}}}>1$.

The absorbing-error analysis proceeds as follows: treating the error state as absorbing, the transient generator has dominant eigenvalue $1-\varepsilon_t$, and $\gamma_t$ governs only the sojourn time in the absorbing state, which does not affect the survival probability. For $\gamma_t=\varepsilon_t$, $\pi_c=1-\varepsilon_t$ and \eqref{eq:markov_valid} reduces to $(1-\varepsilon_t)^L$, recovering the independent model.
\end{proof}

\begin{proof}[Proof of Theorem~\ref{thm:novelty}]
\textbf{Part 1.} By induction: $a$ is admitted only if $d_J(a,a')\ge 1-\tau_{\mathrm{nov}}$ for all $a'\in S_{t-1}$.

\textbf{Part 2.}
Let $K^* = \KL(\hat p_{S_t}\|p_{\theta^*})<\infty$. Then
$K^* \ge \frac{1}{|S_t|}(-\log p_{\theta^*}(a))$ for every $a\in S_t$, giving
$p_{\theta^*}(a) \ge \exp(-K^*|S_t|) =: \delta_t > 0$.
With $|S_t|\ge 2$ distinct elements each receiving mass $\ge\delta_t$, we obtain $\max_a p_{\theta^*}(a)\le 1-\delta_t$.

\textbf{Part 3.}
Let $B\in\{0,1\}$ indicate $a=a^*$. Then $H(p_{\theta^*})\ge H(B)\ge H_{\mathrm{bin}}(\delta_t)>0$ by the data-processing inequality.
\end{proof}

\begin{proof}[Proof of Theorem~\ref{thm:proxy}]
For bivariate Normal $(q,\hat q)$ with Pearson $\rho_{\mathrm{P}}$, the Kruskal--Kendall identity gives $\rho_{\mathrm{S}}=\tfrac{6}{\pi}\arcsin(\rho_{\mathrm{P}}/2)$ \citep{kruskal1958,kendall1948}. Substituting $\rho_{\mathrm{P}}=(1+\SNR^{-1})^{-1/2}$ gives \eqref{eq:rho-snr}. Monotonicity follows because both $\rho_{\mathrm{P}}(\SNR)$ and $\arcsin(\cdot/2)$ are increasing. For \eqref{eq:rho-bound}, Taylor-expand $f(\rho_{\mathrm{P}})=\tfrac{6}{\pi}\arcsin(\rho_{\mathrm{P}}/2)$ at $\rho_{\mathrm{P}}=1$: $f'(1)=2\sqrt{3}/\pi$. Combined with $\rho_{\mathrm{P}}=1-\tfrac{1}{2}\SNR^{-1}+O(\SNR^{-2})$ and concavity of $f$, we obtain $\rho_{\mathrm{S}}\ge 1-(\sqrt{3}/\pi)\SNR^{-1}$.
\end{proof}

\section{Extended Remarks}
\label{app:remarks}

\paragraph{Static corpus as regulariser.}
The static corpus $\mathcal{L}$ acts as a regulariser analogous to a prior in Bayesian CE updates~\citep{deboer2005}, preventing the generation distribution from drifting arbitrarily far from the pre-trained LLM's capabilities.

\paragraph{Quality-monotone projection.}
Assumption~\ref{ass:qmp} holds exactly when $\mathcal{F}$ is the set of all distributions. For LoRA, it requires that the low-rank projection does not decrease mean quality, which holds when the LoRA rank is sufficient to represent the quality-increasing component. Characterising the minimum rank required is an open problem.

\paragraph{LoRA expressiveness ceiling (extended).}
The empirically observed 73--76\% plateau across DeepSeek, Qwen and Mistral is consistent with a LoRA-expressiveness ceiling. Our framework predicts the existence and qualitative shape of such a plateau, and that its level depends on the LoRA rank and the geometry of the elite set, but it does not predict the specific numerical value $C^*\approx 0.74$ from first principles.

\paragraph{Post-concentration regime.}
For $\rho_0=0.50$, Corollary~\ref{cor:rate} predicts $C_t\ge 0.95$ within $t^*\le 5$ cycles. Even after $C_t$ plateaus, continued cycles contribute diverse novel architectures (enforced by the novelty filter). The 22-cycle experiment explores this post-concentration regime, which the theory characterises as a fixed-point maintenance phase.

\paragraph{Direction-of-effect ratio (extended).}
Corollary~\ref{cor:ratio} was reformulated from an absolute gap bound to a ratio bound, correcting the eigenvalue from $(1-\varepsilon)(1-\gamma)/(2-\varepsilon-\gamma)$ to $1-\varepsilon$. Under the corrected model, $\gamma_t$-dependence enters only through $C(\gamma_t,\varepsilon_t)$ which cancels.

\paragraph{Novelty bound $\delta_t$ dependence.}
The constant $\delta_t=\exp(-K^*|S_t|)$ is strictly positive for every finite cycle but decays as $|S_t|$ grows. The novelty filter prevents degenerate mode collapse at any finite cycle but does not guarantee a uniform entropy bound across cycles; the empirical persistence of admissions across all 22 cycles is consistent with this weaker guarantee.

\paragraph{MinHash approximation.}
In practice, $d_J$ is estimated via MinHash signatures of length $k=128$~\citep{broder1997}, yielding an unbiased estimator with standard deviation $O(1/\sqrt{k})$, making the slack negligible at $\tau_{\mathrm{nov}}=0.90$.

\paragraph{Proxy scope and falsifiable predictions.}
\label{rem:proxy-scope}
Theorem~\ref{thm:proxy} delivers two falsifiable predictions: (i)~the LLM with higher SNR should have higher $\rho_{\mathrm{S}}$; (ii)~at very small SNR, $\rho_{\mathrm{S}}$ should be non-significant. The auxiliary inequality \eqref{eq:rho-bound} is non-trivial only when $\SNR>0.55$, which holds for none of our LLMs. Under accuracy-ceiling truncation, Pearson and Spearman can decouple; the Qwen row ($\SNR=0.011$, $\hat\rho=0.635$) is such a case.

\paragraph{DeepSeek ceiling (extended).}
The observed $\hat\rho_{\mathrm{S}}=0.495$ exceeds the population prediction of $\sim 0.10$, consistent with finite-sample fluctuation at $N=12$ and mild deviation from bivariate-Normal assumptions.

\section{Extended Discussion}
\label{app:discussion}

\paragraph{Implications for practice.}
Theorem~\ref{thm:conv} and Corollary~\ref{cor:rate} provide guidance on cycle counts. After $C_t$ plateaus, continued cycles should be justified by novelty-filtered diversity. Theorem~\ref{thm:proxy} provides a diagnostic for proxy reliability: verify $\sigma^2_{\mathrm{arch}}\gg\sigma^2_{\mathrm{noise}}$ before trusting proxy-based rankings.

\paragraph{Reproducibility.}
Per-cycle summary statistics (elite concentration $C_t$, mean quality $Q_t$, and valid-generation rates) are drawn from~\citep{adhikari2026delta}; the corresponding data are available as \texttt{cycle\_data.csv} (66 rows) in the supplementary material. The proxy-reliability study in Section~\ref{sec:proxy-validation} reports new experiments: the top-20 architectures per LLM were fully trained for 50 epochs and paired with their one-epoch proxy scores. Raw per-architecture results are released as \texttt{proxy\_full\_pairs.csv} (43 rows) in the supplementary material. SNR and rank correlations are reproducible using \texttt{scipy.stats.spearmanr} and \texttt{scipy.stats.kendalltau}.

\paragraph{Compute resources.}
The proxy-reliability experiments 
(Section~\ref{sec:proxy-validation}) were conducted on a single 
NVIDIA RTX~4090 GPU (24\,GB VRAM). Fully training 43 architectures 
for 50~epochs each required approximately 4 GPU-hours in total; 
individual runs ranged from under 1~minute to approximately 25~minutes. 
No distributed training was used. Per-cycle LLM generation and 
fine-tuning compute is reported in \citet{adhikari2026delta}.

\paragraph{Fixed threshold.}
The current theory uses a fixed $\tau$. Adaptive thresholds would yield stronger guarantees, as in standard CE~\citep{deboer2005}.

\paragraph{Admission rate and learning signal.}
At $\tau_{\mathrm{nov}}=0.90$, only ${\sim}8.6\%$ of passing Qwen samples are admitted as novel. Higher $\tau_{\mathrm{nov}}$ trades slower per-cycle quality improvement for stronger non-collapse guarantees.

\end{document}